\newcommand{\shortversion}[1]{}
\newcommand{\longversion}[1]{#1}

\documentclass[letterpaper, 10 pt, conference]{ieeeconf}  

\IEEEoverridecommandlockouts                              

\usepackage{amssymb,latexsym,amsfonts,amsmath}
\newtheorem{definition}{Definition}
\newtheorem{theorem}{Theorem}
\usepackage{xkeyval}
\usepackage{color}
\usepackage{tikz}
\usetikzlibrary{automata}
\usetikzlibrary{shapes,arrows}
\usepackage[latin1]{inputenc}
\usepackage{dcolumn}
\usepackage{multirow}
\usepackage{subfigure}
\usepackage{eso-pic}
\usepackage{epsfig}
\usepackage{url}
\usepackage{xcolor}
\usepackage{pgfplots}
\usepackage{pgfplotstable}
\pgfplotsset{compat=1.16}
\usepackage{diagbox}
\usepackage{siunitx}
\usepackage[algosection,linesnumbered,lined, ruled]{algorithm2e}
\usepackage{balance}
\usepackage[switch]{lineno}
\usepackage{textcomp}

\usetikzlibrary{patterns}
\usepgfplotslibrary{external}

\newcommand{\ws}{\mathcal{WS}}

\definecolor{bblue}{HTML}{7FFFD0}
\definecolor{bbblue}{HTML}{0093AF}
\definecolor{bbbblue}{HTML}{483D8B}

\definecolor{yyellow}{HTML}{FADA5E}
\definecolor{yyyellow}{HTML}{FF8F00}
\definecolor{yyyyellow}{HTML}{996515}

\definecolor{ggreen}{HTML}{9BDDA4}
\definecolor{gggreen}{HTML}{29AB87}
\definecolor{ggggreen}{HTML}{195905}

\shortversion{
\renewcommand{\baselinestretch}{0.938}
}

\title{\LARGE \bf
Mobile Recharger Path Planning and Recharge Scheduling \\ in a Multi-Robot Environment}

\author{Tanmoy Kundu$^{1}$ and Indranil Saha$^{2}$
\thanks{*Tanmoy Kundu is supported by Visvesvaraya Ph.D. Fellowship by 
the Department of Electronics and Information Technology, Ministry of Communication and Information Technology, Government of India}
\thanks{$^{1}$ Tanmoy Kundu is with Department of Computer Science and Engineering,
Indian Institute of Technology Kanpur
{\tt\small tanmoy@cse.iitk.ac.in}}%
\thanks{$^{2}$ Indranil Saha is with Department of Computer Science and Engineering,
Indian Institute of Technology Kanpur
{\tt\small isaha@cse.iitk.ac.in}}%
}

\begin{document}

\maketitle
\thispagestyle{empty}
\pagestyle{empty}

\begin{abstract}
In many multi-robot applications, mobile worker robots are often engaged in performing some tasks repetitively by following pre-computed trajectories.
As these robots are battery-powered, they need to get recharged at regular intervals.
We envision that in the future, a few mobile recharger robots will be employed to supply charge to the energy-deficient worker robots recurrently, to keep the overall efficiency of the system optimized.
In this setup, we need to find the time instants and locations for the meeting of the worker robots and recharger robots optimally.
We present a Satisfiability Modulo Theory (SMT)-based approach that captures the activities of the robots in the form of constraints in a sufficiently long finite-length time window (hypercycle) whose repetitions provide their perpetual behavior.
Our SMT encoding ensures that for a chosen length of the hypercycle, the total waiting time of the worker robots due to charge constraints is minimized under certain condition, and close to optimal when the condition does not hold. Moreover, the recharger robots follow the most energy-efficient trajectories. 
We show the efficacy of our approach by comparing it with another variant of the SMT-based method which is not scalable but provides an optimal solution globally, and with a greedy algorithm.
\end{abstract}

\section{Introduction}
\label{sec-intro}


Mobile robots are generally battery powered. They need to recharge their batteries periodically to ensure long-term operation. For example, consider a multi-robot system that has been entrusted with the surveillance responsibility of a large area~\cite{surveillance1,surveillance2}. Each robot has a predefined trajectory that it follows to carry out the surveillance operation. As the robots are battery powered and they are supposed to be operational all the time, there needs to be some mechanism in place to recharge the batteries of the robots whenever required. 
Such a situation also arises when multiple robots are employed for delivering some objects following their pre-planned trajectories in an assembly line~\cite{UnhelkarDBLPSBT18}.


In this paper, we envision a multi-robot service system where a set of 
worker robots (\emph{workers}) follow their respective pre-defined non-intersecting working loops repetitively to carry out their routine work, and a few mobile recharger robots (or \emph{rechargers}) periodically meet the worker robots to recharge their batteries.
This approach is motivated by several mobile charging solutions that have recently appeared in the market~\cite{vw2,SparkCharge}.
\longversion{
Though the mobile rechargers have been developed keeping the electric vehicles in mind, there is a vast potential of this technology to be useful for various multi-robot applications.}

In our proposed system, when a worker goes out of charge, it may have to wait before meeting a recharger. The recharger may take some time to travel to the worker's location. It may even be busy recharging another worker at that time. On the other hand, if the recharger is idle for some time, it may prefer to move to the location of some upcoming recharging ahead of time instead. However, as there are multiple workers, it is challenging to decide towards which worker robot the recharger should move. This decision depends on many factors, such as the distance of the workers from the recharger, lengths of the working loops, the maximum possible charge available to the workers, etc. 
Thus, it is imperative that we \emph{automatically}
synthesize the trajectories of the rechargers in a way that the overall waiting time of all the workers gets minimized, and the rechargers move in the workspace following the most time/energy-efficient trajectories.
We do not consider recharging of the rechargers and assume their recharge requirement to be significantly less frequent than that of the worker robots.


In this paper, we present a Satisfiability Modulo Theory (SMT)~\cite{barrett-smtbookch09} based methodology to decide the initial locations of the rechargers and synthesize their action plans statically, given the working loops of the workers. 
In our approach, we capture the infinite trajectory representing the perpetual behavior of the robot as a finite \emph{hypercycle} whose successive repetitions create the infinite behavior of the workers and the rechargers.
In such a hypercycle, several working loops of the workers can be embedded. 
To be able to repeat a hypercycle, we need to ensure that the initial states of all the robots match with their states at the end of the hypercycle. 
It is also important to decide the initial locations of the rechargers as they have a high impact on the overall efficiency of the system. 
The objective of the synthesis is to minimize the total wait time of all the workers as well as the cost of the movement of the rechargers.

We first attempt to synthesize the trajectories of the rechargers and the recharge schedule of the workers, by reducing the problem into a \emph{monolithic} SMT solving problem. 
However, though this monolithic approach guarantees the optimality of the workers' wait time for a chosen length of the hypercycle, the approach does not scale up well, either with the number of robots
or with the length of the hypercycle. To address this scalability issue, we design an SMT-based two-phase algorithm to solve the problem.
This two-phase algorithm enables us to solve the problems at a larger scale both in terms of the number of robots and the length of the hypercycle. 
We prove that under a certain condition, our two-phase algorithm ensures the optimal wait time of the workers for a chosen length of the hypercycle. 
%

We carry out experiments with up to eight workers and three rechargers. The trajectories have been synthesized within an acceptable time budget (3 hours).
We measure the efficiency of the workers as the proportion of the hypercycle duration during which workers are active \textemdash \ not waiting stand-by to get served by a recharger. 
In most cases, the optimal SMT-based one-shot algorithm faces a timeout, but our two-shot algorithm finds the solution successfully. Moreover, for the instances that the one-shot algorithm can solve, our two-shot algorithm produces the plan with efficiency close to that produced by the one-shot algorithm.
We also compare our SMT-based algorithm with a greedy algorithm. Our SMT-based algorithm achieves $\approx 13-44\%$ better efficiency compared to the greedy algorithm.

In summary, we make the following contributions.
\begin{itemize}
    \item We introduce the mobile recharger path planning problem for a multi-robot system engaged in perpetual activities. Our problem involves both recharge scheduling and path planning of the rechargers to maximize the efficiency of the workers.
    \item We propose an SMT-based solution for the above-mentioned path planning and recharge scheduling problem. Our solution is scalable and produces close to the optimal solution. It also outperforms a carefully crafted greedy algorithm, establishing the efficacy of the SMT-based approach.
    \item We implement our algorithm using the Z3 SMT solver and test the efficiency of the algorithm on various instances of the problem with up to $8$ workers and $3$ rechargers. 
\end{itemize}


\section{Problem}
\label{sec-problem}

\subsection{Preliminaries}
\subsubsection{Workspace ($\ws$)}
In this work, we assume that the robots operate in a 2-D workspace represented as a \mbox{2-D} occupancy grid map.
The grid decomposes the workspace into square-shaped blocks that are assigned unique identifiers to represent their locations in the workspace.
We denote the set of locations in the workspace by $\ws$ and locations covered
by obstacles by $O$. The set of obstacle-free locations in the workspace is $\ws \setminus O$.

\subsubsection{Robot State ($\sigma$)}
The \emph{state} $\sigma$ of a robot consists of 
(a) $\sigma.p$, its position in the workspace, which determines a unique block in the occupancy grid,
(b) $\sigma.v$, its velocity configuration, which represents the current magnitude and direction of the velocity of the robot. We denote the set of all velocity configurations by $V$ and 
assume that it contains a value $v_0$ denoting that the robot is stationary, and
(c) $\sigma.e$, the battery energy available to the robot.


\subsubsection{Motion Primitive ($\gamma$)}
We capture the motion of a robot using a set of \emph{motion primitives} $\Gamma$.
We assume that the robot moves in an occupancy grid in discrete steps of $\tau$ time units.
A motion primitive is a short controllable action that the robot can perform in any time step.
A robot can move from its current location to a destination location by executing a sequence
of motion primitives.

With each motion primitive $\gamma\in\Gamma$, we associate
a \emph{pre-condition} $\mathit{pre}(\gamma)$, which is a formula over the states 
specifying under which conditions a motion primitive can be
executed.
We write $\mathit{post}(\sigma,\gamma)$ for the state the robot attains after executing the motion
primitive $\gamma$ at state $\sigma$.
A motion primitive $\gamma$ causes a displacement of the robot with respect to its current location where the primitive is applied. This displacement is denoted by $\mathit{disp}(\gamma)$. 
Also, each motion primitive $\gamma$ is associated with an energy cost as denoted by $\mathit{cost}(\gamma)$,
which represents the amount of energy spent by the robot while executing the motion primitive.
Thus, if $\sigma' = \mathit{post}(\sigma, \gamma)$, then $\sigma'.p = \sigma.p + \mathit{disp}(\gamma)$
and $\sigma'.e = \sigma.e - \mathit{cost}(\gamma)$.
We use $\mathit{intermediate}(\sigma, \gamma)$ to denote the set 
of grid blocks through which the robot may traverse when $\gamma$ is applied
at state $\sigma$, including the start and end grid blocks.

\subsubsection{Worker ($r_i$) and Recharger ($c_i$) robot}
We consider $n$ workers $R= \{r_1,\ldots, r_n\}$ engaged in some repetitive tasks in the workspace. 
A set of mobile rechargers $C = \{c_1,\ldots, c_m\}$, ideally $m < n$,  are employed for recharging the workers as and when needed and keep the system running uninterruptedly.

The workers move following some predefined trajectories and carry on performing their designated tasks repetitively. 
We assume that the trajectories of the workers do not intersect with each other. That is why we do not need to deal with their collision avoidance.
The \emph{Working loop} for worker ${r_i\in R}$ to perform its designated tasks is denoted by 
${L_i=\langle l^{1}_i,\ldots, l^{|L_i|-1}_i, l^{|L_i|}_i\rangle}$, where the last location on $L_i$ is same as its first location. Each location $l_i^k$ in $L_i$ is associated with a motion primitive $\gamma_i^k \in \Gamma_i$ that enables the worker to move to its next location $l_i^{k+1}$. A state $\sigma_i$ satisfies the precondition $pre(\gamma_i^k)$ if
$\sigma_i.p = l_i^k \wedge \sigma_i.v = v_0$.
We write $\mathit{post}(\sigma_i, \gamma_i^k)$ to denote the state $\sigma_i'$ such that 
\[
    \sigma_i'.p =\left\{
                \begin{array}{ll}
                  l_i^{k+1} \ \text{ if } \sigma_i.p \in\{l_i^1,\ldots l_i^{|L_i|-1} \} \\
                  l_i^1 \ \ \ \ \text{ if } \sigma_i.p = l_i^{|L_i|}
                \end{array}
              \right.
  \]
 $\wedge$ \ $\sigma_i'.v = v_0$ \ $\wedge$ \ $\sigma_i'.e = \sigma_i.e - cost(\gamma_i^k)$.

\noindent
A worker can continue its operation uninterruptedly only if it can get its battery recharged at a regular interval. Worker $r_i$ may stop at any location in $L_i$ for getting recharged. We assume that a worker $r_i$ can recharge its battery only if 
a recharger $c_i$ is positioned somewhere in the neighborhood of $r_i$.
If the worker is at location $p$, its neighborhood is defined as any obstacle-free location 
which is one unit distance (in any direction) away from $p$, i.e.,
$\mathcal{N}(p) = \{p' \mid {p' \in \ws\setminus O} \ \land \ |p'.x - p.x| \leq 1 \ \land \ |p'.y - p.y| \leq 1 \}$.

\subsubsection{Wait Primitive ($\mu$)}
A robot is equipped with a special primitive, called the \emph{wait primitive} $\mu$, that enables it to wait in a location for $\tau$ time units, without causing energy loss. 
called the \emph{wait primitive}, and is denoted by $\mu$.
A state $\sigma$ satisfies the precondition $\mathit{pre}(\mu)$ if $\sigma.v = v_0$.
We write $\mathit{post}(\sigma, \mu)$ to denote the state $\sigma'$ such that 
 $\sigma'.p = \sigma.p$ \ $\wedge$ \ $\sigma'.v = v_0$ \ $\wedge$ \ $\sigma'.e = \sigma.e$ ($cost(\mu)$ = $0$).
 Moreover, $\mathit{intermediate}(\sigma, \mu)$ = $\{\sigma.p\}$.

\subsubsection{Recharge Primitive ($\nu$)}
\label{subsubsec-rechargeprim}
The workers are equipped with a \emph{recharge primitive} ($\nu$). A worker robot $r_i$ can apply $\nu$ primitive to recharge its battery to the maximum possible energy $emax_i$.
Like a motion primitive, $\nu$ is also associated with a precondition $\mathit{pre}(\nu)$ and a postcondition $\mathit{post}(\sigma,\nu)$. The precondition and postcondition depend on the specific recharge strategy.
In this work, we consider the most flexible recharge strategy where a worker robot is allowed to get recharged whenever its battery charge is not full, and it can be recharged by any recharge amount not necessarily up to its full capacity.
For this recharge strategy, a state $\sigma$ satisfies the precondition $\mathit{pre}(\nu)$ if $\sigma.e < emax_i \ \wedge \ \sigma.v = v_0  \wedge \ recharger$, where $recharger$ is a proposition which becomes $true$ when the robot $r_i$ has access to a recharger.
The postcondition $\mathit{post}(\sigma, \nu)$  denotes the state $\sigma'$ such that 
$\sigma'.p = \sigma.p$ \ $\wedge$ \ $\sigma'.v = v_0$ \ $\wedge$ \ $\sigma'.e = \sigma.e + \delta$ $\wedge$ \ $0 < \delta \le \delta_{max}$ \ $\wedge$ \ \ $\sigma'.e \le emax_i$, \ 
$\delta \in \mathbb{R}^+$, where $\delta_{max}$ is the maximum recharge amount per $\tau$ time units. 
Moreover, $\mathit{intermediate}(\sigma, \nu)$ = $\{\sigma.p\}$.


\subsubsection{Action Plan ($\rho$) and Trajectory ($\sigma$)}
\label{subsubsec-motionplan}
We capture the run-time behavior for a robot by a discrete-time transition system.
Let $\sigma_1$ and $\sigma_2$ be two states of the robot.
For a primitive $\rho \in \Gamma \cup \{\mu,\nu\}$, $\sigma_1 \xrightarrow{\rho} \sigma_2$ is a valid transition
iff $\sigma_{1} \models \mathit{pre}(\rho)$, $\sigma_{2} \models \mathit{post}(\sigma_{1}, \rho)$, and
$ \mathit{intermediate}(\sigma_1, \rho) \ \cap$ $O = \emptyset$.

The \textit{action plan} for a robot is defined as a sequence of primitives to be applied to the robot 
to move it in a way that its objective is achieved while satisfying various constraints. 
An action plan is denoted by a (potentially infinite) sequence of primitives $\rho = (\rho_1 \rho_2 \ldots)$, where $\rho_i \in \Gamma \cup \{\mu, \nu\}$ for all $i \in \{1, 2, \ldots \}$. 
The rechargers do not have any recharge primitive $\nu$.

Given the current state $\sigma_0$ of some robot and an action plan \mbox{$\rho=(\rho_1\rho_2\rho_3\ldots)$}, the  \textit{trajectory} of the robot is given 
by \mbox{$\sigma= (\sigma_0 \sigma_1 \sigma_2 \ldots)$} such that for all $i \in \{1,2,3,\ldots\}$, ${\sigma_{i-1} \xrightarrow{\rho_i} \sigma_i}$.

\subsubsection{Hypercycle ($T$)}

We synthesize the recharge schedules for the workers and the trajectories for the rechargers in a \textit{hypercycle}, which is a time window of $T$ units.
We abuse the notation slightly and denote both the hypercycle and its length by $T$. 
Successive repetitions of $T$ essentially creates a long (potentially infinite) execution of the system. Hypercycle $T$ is a parameter in our algorithm. As will be clear later, the efficiency of the worker robots increases with the value of $T$, but the computation time of our algorithm also increases with $T$. 

To be able to repeat any number of hypercycles, the states (location and charge level) of the robots at the beginning of a hypercycle should match with that at the end of the hypercycle.
During the time window $T$, the trajectory of worker robot $r_i$ 
is denoted by $S_i=\langle \sigma^{1}_i,\ldots, \sigma^{T}_i\rangle$ and the trajectory of recharger $c_i$ is $\hat{S}_i=\langle \hat{\sigma}_i,\ldots, \hat{\sigma}_i^{T}\rangle$, where \\$\forall r_i \in R : \sigma_i^T=\sigma_i^1$, and $\forall c_i \in C: \hat{\sigma}_i^T=\hat{\sigma}_i^1$. The worker robots start their operation with full charge, i.e., $\sigma_i^1.e = emax_i$.
Once the worker robots are at their final locations, they need to be recharged to full charge to accomplish the state matching, i.e., $\sigma_i^T.e = emax_i$.

During $T$, every worker robot $r_i$ completes several rounds of its working loop $L_i$.
Worker robot $r_i$'s trajectory $S_i$ is composed of multiple concatenations of working loop $L_i$ of worker $r_i$.
Also, due to recharging and idle waiting, some of the trajectory points in $L_i$ may be repeated, as 
as the robot does not change its position during those events. We represent the $j^{th}$ extended working loop as $L_i^j$, where zero or multiple repetitions of trajectory points may occur. Thus, $S_i$ is composed of multiple (say $g$) concatenations of $L_i^j$s :
$$S_i \equiv \langle \sigma^{1}_i.p,\ldots, \sigma^{T}_i.p\rangle \equiv \underbrace{L_i^1\circ L_i^2 \circ \ldots \circ L_i^g}_{g \ \text{times}}$$
where $\circ$ is the concatenation operator. Also, ${\sigma}_i^1.p = {\sigma}_i^T.p = l_i^1$.

Also, recharger $c_j$'s trajectory $\hat{S}_j$ is synthesized by our algorithm as described later.


 The finite length trajectories during $T$ for all the workers and rechargers are captured as 
 $$S=\langle S_1, S_2, \ldots, S_{|R|}, \hat{S}_1, \hat{S}_2, \ldots, \hat{S}_{|C|}\rangle.$$


\subsection{Problem Definition}
\label{subsec-probdef}


In this section, we define the problem formally.
The inputs to the problem are the working loops ${L_i=\langle l^{1}_i,\ldots, l^{|L_i|-1}_i, l^{|L_i|}_i\rangle}$ with corresponding sequence of motion primitives $\langle \gamma^{1}_i,\ldots, \gamma^{|L_i|-1}_i, \gamma^{|L_i|}_i\rangle$ and maximum energy $emax_i$ for each worker $r_i$, the set of motion primitives $\Gamma_i$ and $\hat{\Gamma}_i$ for each worker $r_i$ and each recharger $c_i$, the maximum recharge amount $\delta_{max}$, and a set of \emph{potential} initial locations $P \subseteq \ws\setminus O$ for the rechargers.

At each time step of $T$, the workers and the rechargers perform some actions. We have to decide the actions of the robots at every time step in order to find the optimal solution to the planning problem. The action plan for worker $r_i$ during time window $T$ is defined as:
$\rho_i = \langle \rho^{1}_i,\ldots, \rho^{T-1}_i \rangle$, 
where $\rho^{t}_i \in \Gamma_i \cup \{\mu,\nu\}$. 
Similarly, the action plan for recharger $c_i$ is defined by 
${\hat{\rho}_i = \langle \hat{\rho}^{1}_i,\ldots, \hat{\rho}^{T-1}_i\rangle}$,
where $\hat{\rho}_i^t \in \hat{\Gamma}_i \cup \{\mu\}$.
For the whole system of robots, the consolidated action plan is defined as $\rho = \langle \rho_1, \rho_2,\ldots, \rho_{|R|}, \hat{\rho}_1, \hat{\rho}_2,\ldots, \hat{\rho}_{|C|}\rangle$.
We denote by $\delta_i^t$, $0<\delta_i^t\le \delta_{max}$, the amount of energy used in the recharge of worker $r_i$ at the $t$-th time instant if $\rho^{t}_i = \nu$. 


We now formulate the problem as an \emph{optimization problem}.
The decision variables for this optimization problem are the action plans of the workers and the rechargers, the  recharge amount at every recharge instance, and the initial location of the rechargers decided from a given set of potential initial locations $P$. 

	\begin{itemize}
		\item 	$\forall t \in \{1, \ldots, T-1\}. \rho^t = \langle \rho_1^t, \ldots, \rho_{|R|}^t, \hat{\rho}_1^t, \ldots, \hat{\rho}_{|C|}^t \rangle$,
		\item 	$\forall t \in \{1, \ldots, T-1\}. \forall r_i \in R. \ \ \delta_i^t\  \text{when }\rho_i^t=\nu$,
		\item 	$\forall c_j \in C. \ \ \hat{\sigma}_j^1.p \in P$.
	\end{itemize}	

The objective of the optimization problem is to minimize the total waiting time $\mathcal{W}$ of the workers, and the travel cost $\mathcal{U}$ of the recharger, with weights $w_1$ and $w_2$ respectively:
$$\mathit{Minimize}\left(w_1 \cdot {\mathcal{W}} + w_2 \cdot \mathcal{U}\right )$$
where
$$\mathcal{W}:= \sum_{t\in T-1} \sum_{r_i\in R} [\rho_i^t = \mu], \ \ \ \  \mathcal{U}:=\sum_{t\in T-1} \sum_{c_i\in C} \hat{\rho}_i^t.cost,$$
$[\rho_i^t = \mu]$ is $1$ if $\rho_i^t = \mu$ and $0$ otherwise.

The above optimization problem has to be solved under several constraints, as introduced in Section~\ref{sec-algo}. 

To measure the performance of a planning algorithm, we introduce the following notion of efficiency of a recharge plan for a multi-robot system.
\begin{definition}{\emph{Efficiency.}} The efficiency $E$ of a multi-robot system, under a recharge plan, is defined as the percentage of time spent by the workers doing their tasks or recharging (excluding wait time to get recharged), during time window $T$. Mathematically,
\begin{equation}
E = \frac{|R| \cdot T - \mathcal{W}} {|R| \cdot T} \times 100.
\label{eq-efficiency}
\end{equation}
\end{definition}



\shortversion{
We illustrate the problem with an example which is available in the full version of the paper~\cite{KunduS21}.
}

\longversion{
\subsection{Example}

\begin{figure}[t!]
{
\begin{center}
\subfigure[]{\resizebox{0.325\linewidth}{!} { \includegraphics[scale=1]{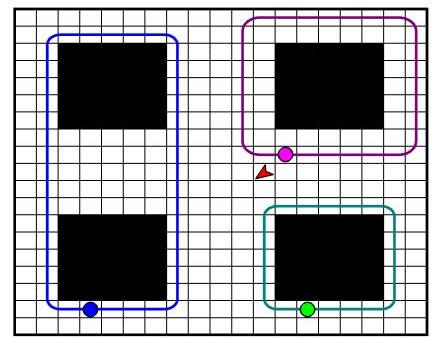}}}
\subfigure[]{\resizebox{0.325\linewidth}{!} {\includegraphics[scale=1]{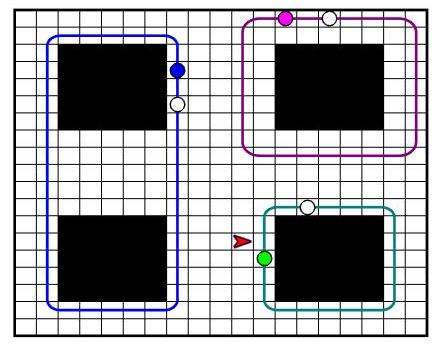}}}
\subfigure[]{\resizebox{0.325\linewidth}{!} { \includegraphics[scale=1]{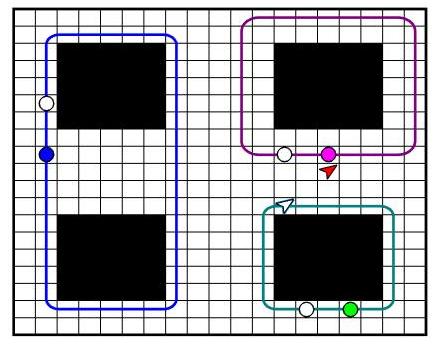}}}
\subfigure[]{\resizebox{0.325\linewidth}{!} {\includegraphics[scale=1]{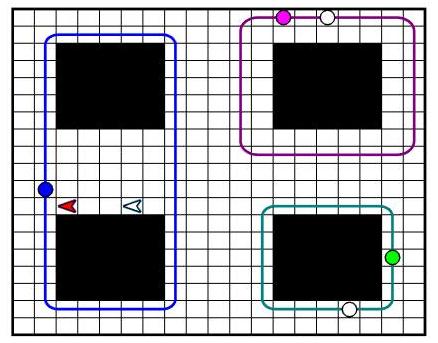}}}
\subfigure[]{\resizebox{0.325\linewidth}{!} { \includegraphics[scale=1]{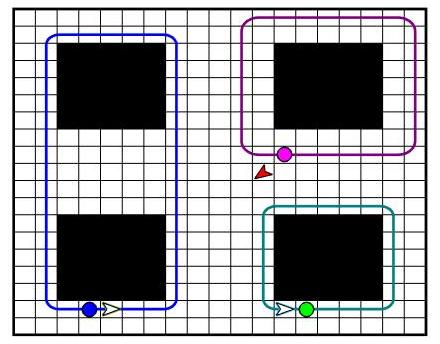}}}
\subfigure[]{\resizebox{0.315\linewidth}{!} {\includegraphics[scale=1]{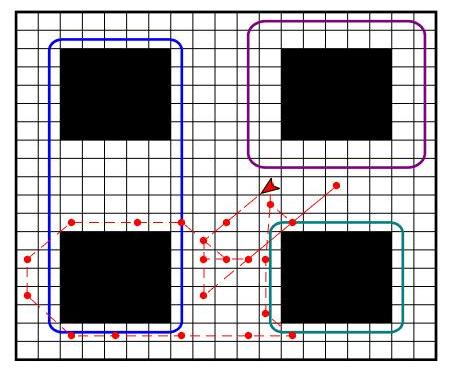}}}
\end{center}
}
\caption{{\small Snapshots of plan execution of a recharger (arrowhead) for three workers (circles): (a) initial position,
(b) the recharger waits for the green worker, (c) both the red worker and the recharger move together, (d) the blue worker waits for the recharger, (e) the robots return to the same initial and final positions, (f) recharger's trajectory.}
}
\label{fig:motivating-example}
\end{figure}

We present a simple example involving three worker robots and one recharger robot to illustrate the trajectories generated by our algorithm (Figure~\ref{fig:motivating-example}). The workers are shown as circles and the recharger as an arrow-head. To recharge a worker, a recharger has to be placed in some neighboring location of the worker for one or more consecutive time steps. The present states of the robots are colored, and the previous states are shown as hollow circles. 
The initial positions of the robots are shown in  Figure~\ref{fig:motivating-example}(a).
At time points 8 and 9, the recharger places itself at some suitable location even before the worker on the green loop goes out of charge (Figure~\ref{fig:motivating-example}(b)).
At time steps 13 and 14, to recharge the worker on the red loop, both the worker and the recharger move simultaneously towards each other to perform recharging for two consecutive time steps (Figure~\ref{fig:motivating-example}(c)).
At time steps 35 and 36, the worker on the blue loop has to wait before the recharger can come close to it for recharging (Figure~\ref{fig:motivating-example}(d)). Within 40 time steps, the workers are back to their respective initial locations; however, their charge levels do not match with their initial charge levels. Also, the current location of the recharger is not the same as its initial location. Few more time steps (in this case, 8) are needed to restore the workers to their initial charge levels. From time step 41 onwards, the recharger moves to each worker to boost them to their initial charge levels and places itself back to its initial location. Thus, the entire system's final state matches exactly with the initial state at time step 48 (Figure~\ref{fig:motivating-example}(e)). Hence, for this example, 48 is the length of the \textit{extended hypercycle} when the length of the \textit{original hypercycle} is 40 (These terms will be introduced formally in the next section).
Figure~\ref{fig:motivating-example}(f) shows the trajectory of the recharger in the extended hypercycle. In the following section, we will provide an SMT-based methodology to generate such plans for mobile rechargers automatically.
}
\section{Algorithm}
\label{sec-algo}

\subsection{One-Shot Algorithm}

We first present a naive approach to solve the planning problem and use it as the baseline to evaluate our proposed Two-shot algorithm discussed later. 
In this approach, we formulate the problem as a monolithic (one-shot) optimization problem with the objective function and the decision variables described in the problem definition above. Here we present the constraints of the optimization problem.

\medskip
\noindent
\textit{Constraints for the workers:}
For worker $r_i$, we denote the constraints as:
$$|[Worker_i]| = I(\sigma_i^1) \wedge \displaystyle\mathop\bigwedge_{t = 1}^{T-1} Tr(\sigma_i^{t}, \rho_i^{t}, \sigma_i^{t+1}) \wedge F(\sigma_i^T) \ .$$ 
Here, $I$ is the initial location constraint of $r_i$ which is a predefined location on the working loop, and the initial charge of $r_i$ is equal to its full charge capacity.  
$$I(\sigma_i^1) \equiv  \sigma_i^1.p =l_i^1.p\ \wedge \ \sigma_i^1.v= v_0\  \wedge \ \sigma_i^1.e= emax_i$$
The transition constraint $Tr$ of any worker is based on three types of primitives --- motion, wait and recharge; and that the battery charge at any time point is between zero and maximum (full) charge.
\begin{flalign*}
Tr(\sigma_i^{t},\rho_i^{t},\sigma_i^{t+1}) & \equiv 0 \le \sigma^t_i.e \le emax_i \ \wedge \ \rho_i^{t} \in \Gamma_i \cup \{\mu, \nu\} & \\
& \wedge \ \sigma_i^t \models \mathit{pre}(\rho_i^{t}) \ \wedge \ \sigma^{t+1}_i \models \mathit{post}(\sigma_i^{t}, \rho_i^{t}) &
\end{flalign*}
To connect the trajectory of worker $r_i$ with its working loop $L_i$, we introduce a variable $\theta_i^t$, $1 \le t \le T-1$, that maps $r_i$'s position to an appropriate trajectory point on $L_i$, at $t^{th}$ time step. To achieve this, we conjunct 
$\theta_i^1 = 1\ \wedge\ \rho_i^1 = \gamma_i^1$ with $I(\sigma_i^1)$ and 
the following constraints with $Tr(\sigma_i^{t},\rho_i^{t},\sigma_i^{t+1})$ :
\begin{align*}
&\rho_i^t \in \{\mu,\nu\} \rightarrow \theta_i^{t+1} = \theta_i^t\\
&(\rho_i^t \in \Gamma_i \ \wedge \ \theta_i^t < |L_i|) \rightarrow  (\rho_i^t = \gamma_i^{\theta_i^t} \ \wedge \ \theta_i^{t+1} = \theta_i^t + 1)\\
&(\rho_i^t \in \Gamma_i \ \wedge \ \theta_i^t = |L_i|) \rightarrow (\rho_i^t = \gamma_i^{|L_i|} \ \wedge \ \theta_i^{t+1} = 1
)
\end{align*}
Note that the sequence of motion primitives corresponding to the working loop of $r_i$ is $\langle \gamma^{1}_i,\ldots, \gamma^{|L_i|-1}_i, \gamma^{|L_i|}_i\rangle$, and $\theta_i^t$ keeps track of the index of the motion primitive that has to be applied in the current step $t$.

The final state constraint $F$ captures that worker's location and charge level should match at the final and initial time points of the hypercycle: 
$$F(\sigma_i^T) \equiv\ \sigma_i^T.p = \sigma_i^1.p\ \wedge \ \sigma_i^T.v= v_0\ \wedge\ \sigma_i^T.e=\sigma_i^1.e.$$

\medskip
\noindent	
\textit{Constraints for the rechargers:}
The constraints for a recharger $c_j$ in hypercycle $T$ are given by\\
$$|[Recharger_j]| = \hat{I}(\hat{\sigma}_j^1)\ \wedge\ \displaystyle\mathop\bigwedge_{t = 1}^{T-1} \widehat{Tr}(\hat{\sigma}_j^t, \hat{\rho}_j^{t}, \hat{\sigma}_j^{t+1})\ \wedge\  \hat{F}(\hat{\sigma}_j^T).$$
The initial constraint $\hat{I}$ sets the initial location of recharger $c_i$ from the set of potential initial locations $P$, which is found by the solver:  
$$\hat{I}(\hat{\sigma}_j^1) \equiv\  \hat{\sigma}_j^1.p \in P \ \wedge \ \sigma_j^1.v= v_0.$$
Transition constraint $\widehat{Tr}$ of any recharger $c_j$ is based on two types of primitives --- motion and wait; state transition satisfies the precondition and
postconditions described above. Also, recharger $c_j$'s transition captures obstacle avoidance and collision avoidance with workers and other rechargers.

\noindent
\begin{flalign*}
&\widehat{Tr}(\hat{\sigma}_j^t, \hat{\rho}_j^{t}, \hat{\sigma}_j^{t+1}) \equiv
\hat{\rho}_j^{t} \in \Gamma_j \cup \{\mu\}\ \wedge  &&\\
&\hat{\sigma}_j^t \models \mathit{pre}(\hat{\rho}_j^{t})\ \wedge\ \hat{\sigma}_{t+1} \models \mathit{post}(\hat{\sigma}_j^{t}, \hat{\rho}_j^{t})\ \wedge && \\
&\textit{intermediate}(\hat{\sigma}_j^t, \hat{\rho}_j^{t}) \notin O\ \wedge && \\ 
&\textit{intermediate}(\hat{\sigma}_j^t, \hat{\rho}_j^{t}) \ \cap
\bigcup\limits_{r_j \in R} \textit{intermediate}(\sigma_j^t, \rho_j^{t}) = \emptyset\  \wedge && \\
&\textit{intermediate}(\hat{\sigma}_j^t, \hat{\rho}_j^{t}) \ \cap
\bigcup\limits_{c_k \in C \setminus \{c_j\}} \textit{intermediate}(\hat{\sigma}_k^t, \hat{\rho}_k^{t}) = \emptyset  &&
\end{flalign*}
Collision avoidance constraints ensure that a recharger does not occupy the same cell occupied by any worker or any other recharger at any time step $t$.

Final state constraint $\hat{F}$ at time point $T$ matches the final location of a recharger with its initial location. We do not consider the charge level of rechargers in this paper. 
$$\hat{F}(\hat{\sigma}_i^T) \equiv\  \hat{\sigma}_i^T.p = \hat{\sigma}_i^1.p \ \wedge \hat{\sigma}_i^T.v = v_0$$

\medskip
\noindent	
\textit{Constraints for the entire system:}
The constraints for the workers and the rechargers collectively make up the constraints for the entire system. 
$$\resizebox{0.475\textwidth}{!}{ $|[System]| = \left(\displaystyle\mathop{\bigwedge}_{r_i \in R} |[Worker_i]|\right) \wedge \left(\displaystyle\mathop{\bigwedge}_{c_j \in C} |[Recharger_j]|\right)$}$$

The one-shot algorithm synthesizes the trajectory of the recharger minimizing the waiting time of the workers. As we solve the planning problem for a fixed number of worker robots and a fixed length of the hypercycle, from Equation~\eqref{eq-efficiency}, we can ensure that the algorithm provides a plan with the maximal efficiency.
This discussion leads to the following theorem.
\begin{theorem}
For a given length of hypercycle $T$, if the one-shot algorithm is solved with weights $w_1 = 1$ and $w_2 = 0$ in the objective function introduced in Section~\ref{subsec-probdef}, then it produces a plan with maximum efficiency.
\end{theorem}

Though the one-shot algorithm provides a plan with optimal efficiency, it suffers from a \emph{lack of scalability}, as shown in the experimental results.

\subsection{Two-Shot Algorithm}
\label{subsec-twoshot}
To address the above scalability issue, we design an algorithm by splitting the problem into two phases.
 
\subsubsection{First phase}
\label{subsubsec-twoshot-first}
The duration of the first phase is equal to the length of the original hypercycle $T$. In this phase, the workers traverse the maximum possible number of working loops, and within $T$ time duration, they return to their respective initial locations. This phase handles matching of \textit{workers' start and end locations only}; charge level matching or rechargers' location matching are not handled in this phase.
During $T$, a worker may require several \textit{intermediate rechargings}, and it returns to its initial location where its charge is below its initial (full) charge.

\medskip
\noindent 
\textit{Objective:}
In this phase, we optimize only the total waiting time of the robots, i.e., minimize $\mathcal{W}$ only.

\medskip
\noindent
\textit{Constraints for worker and recharger robots:}
Constraints for the worker and recharger robots are the same as that of the One-shot algorithm discussed above, with a few exceptions.
For worker $r_i$, the final state constraints $F'$ do not capture charge matching (with the initial state) at the end of $T$. Therefore,  $F'(\sigma_i^T)\equiv\ \sigma_i^T.p=\sigma_i^1.p \ \ \wedge \ \ \sigma_i^T.v=v_0$. For recharger $c_j$, we do not need to enforce any restriction (including location matching) on its final state at the end of $T$. 

\medskip
\noindent 
\textit{Outcomes forwarded to the second phase:}
The following outcomes are forwarded to the second phase of our algorithm.

\textit{(i) Intermediate recharging instances:} This refers to the time point and location of every recharging occurred during $T$.
For each recharger $c_j$, we store the set of $\langle \tau, p\rangle$ (recharging instance) such that $c_j$ recharges some worker $r_k$ at time point $\tau$ and location $p$:
$$\resizebox{0.475\textwidth}{!}{$\eta_j=\{\langle \tau, p\rangle \ \big|\ \rho_j^\tau=\mu\ \land\ \exists r_k\in R,\  \rho_k^\tau=\nu\ \land\ \hat{\sigma}_j^\tau.p \in \mathcal{N}(\sigma_k^\tau.p) \}.$}$$

\textit{(ii) State of the workers:} After location matching of the workers in the first phase, we record the time point ($\tau$) when a worker halts finally, and the number of recharge instances ($d$) required to recharge it fully. We capture this information in $\zeta_i$ for worker robot $r_i$ :
$$
\resizebox{0.475\textwidth}{!}{$\zeta_i  =\langle \tau,d\rangle:
 \forall t: \tau \le t \le T. \ \sigma_i^t.p  = \sigma_i^1.p \land
d = \left\lceil \frac{emax_i - \sigma_i^\tau.e}{\delta_{max}}\right\rceil$}$$

\textit{(iii) Initial location of the rechargers:} For each recharger $c_j$, its initial location $\hat{\sigma}_j^1.p \in P$.

\subsubsection{Second phase}
In this phase, we minimally extend the length of the original hypercycle ($T$) to meet the remaining matching constraints, viz., charge level matching of workers, and location matching of rechargers. 
We denote the duration of the extended hypercycle by $T' (> T)$. 
For the time duration $(T'-T)$, we synthesize the action plan of the rechargers such that they move to the workers to recharge them up to their initial (full) charge, and then they return to their respective initial locations. Thus, all the matching requirements are fulfilled at the end of $T'$, after the second phase.


In this phase, we essentially synthesize trajectories for the rechargers, from time point $1$ till $T'$, with some waypoints already received from the first phase. These waypoints are the initial and final locations (which are the same) of the rechargers and the time instants and duration for intermediate recharging. 


\medskip
\noindent 
\textit{Objective:}
In this phase, we optimize the total cost of the recharger robots, i.e., the objective is to minimize $\mathcal{U}$.

\medskip
\noindent
\textit{Constraints for workers:}
Remember that $\zeta_i.\tau$ is the time point when worker $r_i$ returns to its initial location during the first phase. Between time points $\zeta_i.\tau$ and $T'$ there is a time instant $t$ when some recharger $c_j$ starts recharging worker $r_i$; and this goes on for duration $\zeta_i.d$ thereafter. As obvious, recharger $c_j$ has to be placed at the neighborhood $r_i$ for duration $\zeta_i.d$ :
\begin{align*}
|[Worker_i'']| \equiv & \  \exists t, \ \zeta_i.\tau \le t \le T'-\zeta_i.d+1 \text{ and } \exists c_j \in C\\
\hat{\rho}_j^t= \mu & \ \wedge\ \hat{\rho}_j^{t+1}=\mu\ \wedge \ldots\ \wedge\  \hat{\rho}_j^{t+\zeta.d-1}=\mu\ \wedge\\
\hat{\sigma}_j^t.p= &\hat{\sigma}_j^{t+1}.p=\ldots= \hat{\sigma}_j^{t+\zeta_i.d-1}.p \in \mathcal{N}(\zeta_i.p) 
\end{align*}

\medskip
\noindent
\textit{Constraints for rechargers:}
The constraints for recharger $c_j$ are:
\begin{align*}
|[Recharger_j'']| \equiv & \ \hat{I''}(\hat{\sigma}_j^1)\ \wedge\ \displaystyle\mathop\bigwedge_{t = 1}^{T'-1} \widehat{Tr''}(\hat{\sigma}_j^t, \hat{\rho}_j^{t+1}, \hat{\sigma}_j^{t+1})\\ 
&\ \wedge\ \hat{F''}(\hat{\sigma}_j^T)\ \wedge\ \widehat{IR}(\eta_j)
\end{align*}

\noindent
In this phase, $T'(>T)$ is the last time point of the extended hypercycle. Constraint $I''$ sets the initial location of recharger $c_j$, as received from the first phase, i.e., \mbox{$\hat{I''}(\hat{\sigma}_j^1) \equiv \hat{\sigma}_j^1.p = \hat{l}_j^1$}.
Constraint $F''$ matches the last location of $c_j$ with its initial location, i.e., \mbox{$\hat{F''}(\hat{\sigma}_j^{T'}) \equiv\ \hat{\sigma}_j^{T'} = \hat{\sigma}_j^1$}.
Constraint $\widehat{IR}$ handles intermediate recharging (and associated rules) of workers during the original hypercycle $T$ based on $\eta_j$ which is already received from the first phase:
\mbox{\scalebox{0.96}{$\widehat{IR}(\eta_j) \equiv \forall t\in T: \langle t, p \rangle \in \eta_j \implies \hat{\sigma}_j^t.p \in \mathcal{N}(p)\ \land\ \hat{\rho}_j^t=\mu$}}.
Thus, at the end of the second phase, our algorithm meets all three necessary matchings to enable repetitions of the same extended-hypercycle $T'$ for arbitrary number of times.

\medskip
\noindent
\textit {Optimality of $T'$:} After the first phase, we run a loop in which $T$-value is increased by one in every iteration and is assigned to $T'$. In every iteration, the constraints are checked for satisfiability. Whenever they are satisfied for the first time, the algorithm terminates, and we get a plan for the minimal value of $T'$. 

For a given $T$, the two-shot algorithm is not guaranteed to produce a plan with optimal efficiency. However, the following theorem establishes the conditional optimality of the two-shot algorithm.

\begin{theorem}
For a given original hypercycle length $T$, if for all the robots $r_i \in R$, the length of the working loop $L_i$ is strictly greater than $T'- \zeta_i.\tau$, i.e., $\forall r_i \in R. \ |L_i| > T'- \zeta_i.\tau$, then the solution produced by the two phase algorithm ensures maximal working efficiency.
\end{theorem}
\begin{proof}
In the first phase of the algorithm, we ensure that any worker robot $r_i$ traverses its working loop for the maximal number of times. 
The robot $r_i$ stops its operation at  time $\zeta_i.\tau \le T$. If the length of the working loop $L_i$ is strictly greater than $T'- \zeta_i.\tau$ for any robot $r_i$, it is not possible to include any more working loop for any of the robots in the extended hypercycle of length $T'$. This ensures the maximal working efficiency of the generated plan.
\end{proof}

\shortversion{
Our algorithm generates the trajectories of the mobile rechargers based on the assumption that the workers and the rechargers move in lock steps. In the full version~\cite{KunduS21}, we discuss how the delay uncertainty in the movement of the robots can be address while implementing our algorithm for a real multi-robot system.
}

\longversion {
\subsection{Mechanism for Dealing with Delay Uncertainty}
Our algorithm generates the trajectories of the mobile rechargers based on the assumption that the workers and the rechargers move in lock steps. However, in reality, the robots cannot move synchronously due to the delay uncertainty in their motion.
To deal with this delay uncertainty, we can employ the following measures during the execution of the statically computed plan.
We assume that the workers and rechargers are fully aware of the recharge plan. They can utilize their knowledge about the recharge plan to deal with any run-time anomaly due to delay uncertainty.
First, at any recharge point, both the worker and the recharger wait if the other robot does not arrive at the designated location at the anticipated time.
A worker and a recharger do not leave the recharge location before the recharge is initiated and completed successfully.
Second, to ensure that the generated trajectories do not become useless after the execution of a few hypercycles, a synchronization event can be conducted after the completion of each hypercycle. 
Once a recharger reaches its initial location, it broadcasts a \emph{sync} message to all other robots (both workers and other rechargers). 
When a robot (either worker or recharger) receives exactly $|C|$ \emph{sync} messages, it starts its operation for the next hypercycle. We assume that the communication is reliable and requires negligible time.
}

\subsection{A Greedy Algorithm}
\label{sec-greedy}

To demonstrate that our SMT-based algorithm is indeed essential for obtaining a superior solution to the mobile recharger path planning problem, we design a baseline greedy algorithm for the sake of comparison. In this algorithm, we ensure that when a recharger becomes available, it moves towards a location where it will get the opportunity to recharge an energy-deficient robot at the earliest, i.e., it moves to the nearest charge-deficient worker.

At any time step $t$,  for each worker $r \in R$ and for each recharger $c \in C$, we compute $\lambda_{cr}$ that captures the duration after which recharger $c$ can start recharging robot $r$.
There could be the following two cases: (i) If worker $r$ is already devoid of charge and has become stationary, then $\lambda_{cr}$ gives the time required for an available recharger $c$ to move to the current location of worker $r$.
(ii) If robot $r$ is currently (at time $t$) moving, then $\lambda_{cr}$ represents the \emph{maximum} of the following: (a) time required for robot $r$ to become energy deficient, and (b) time required for recharger $c$ to reach the final location of robot $r$.
Once we compute $\lambda_{cr}$ for all $c \in C_t$ and all $r\in R$, we choose the $\langle r^*, c^* \rangle$ pair for which $\lambda_{cr}$ is the minimum.
\shortversion{
We implement the greedy strategy over a hypercycle like we do in our SMT-based approach. 
}

\longversion{
We implement the greedy strategy over a hypercycle like we do in our SMT-based approach. 
This greedy algorithm can be used to design an online mechanism for mobile recharging as well.
However, the online implementation requires an infrastructure for keeping track of the state of the system and running the scheduling algorithm at appropriate time instants.
}

\section{Evaluation}
\label{sec-expr}


\subsection{Experimental Setup}

\begin{figure}[t!]
{
\begin{center}
\subfigure[]{\resizebox{0.24\linewidth}{!} { \includegraphics[scale=1]{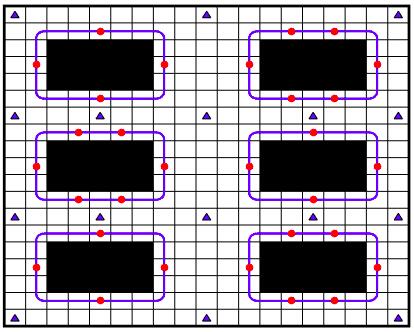}}}
\subfigure[]{\resizebox{0.24\linewidth}{!} {\includegraphics[scale=1]{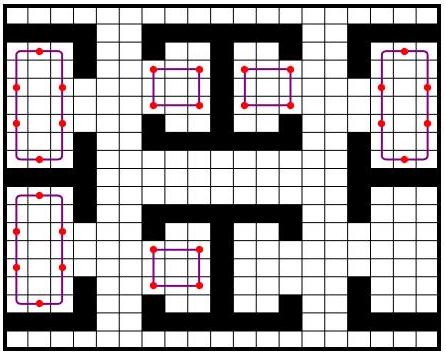}}}
\subfigure[]{\resizebox{0.236\linewidth}{!} {\includegraphics[scale=1]{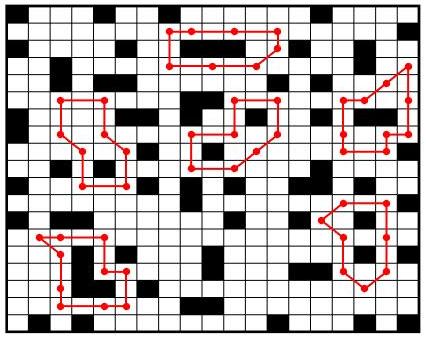}}}
\subfigure[]{\resizebox{0.236\linewidth}{!} {\includegraphics[scale=1]{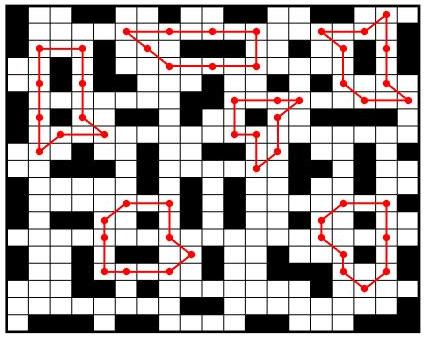}}}
\end{center}
}
\caption{{\small Workspaces for experiments: (a) Warehouse, (b) Artificial floor, (c) Random with 20\% obstacles, (d) Random with 30\% obstacles. Loops represent working loops of the worker robots. Red dots represent the way-points of worker's trajectory. Blue triangles in (a) show the potential initial locations for the rechargers in the workspace.}
}
\label{fig:workspaces}
\end{figure}

The workspaces ($19\times 19$ dimension) used for our experiments are shown in Figure~\ref{fig:workspaces}.
We consider all the robots to follow motion primitives of differential-drive robots like Turtlebot~\cite{turtlebot}.
If fully recharged, half of the workers go out of charge in $10$ time steps, and the remaining workers go out of charge in $12$ time steps. Note that each step may correspond to covering a long distance. The value of $\delta_{max}$ (the maximum amount of recharge per time unit) is chosen as $10$ units.
The bar plots representing the efficiency are divided into two parts -- the lower part gives the efficiency due to movement (work), and the upper one represents the efficiency contribution due to getting recharged (recharge). 
If a bar shows a value of $x\%$, it implies that $(100-x)\%$ of the time the worker has spent in waiting idly for meeting a recharger.
For all our experiments, the timeout is set to $3\si{\hour}$.

Our experiments were carried out in a system with i7-6500U CPU @ 2.50GHz and 16 GB RAM. We use Z3~\cite{Z3} as the back-end SMT solver. 
We have also carried out the experiments using Gurobi optimizer~\cite{gurobi} by modeling our problem as an appropriate Integer Linear Programming problem. However, in our experiments, Z3 consistently outperformed Gurobi in terms of computation time. Thus, we present our results using Z3 as the back-end solver only.
In our SMT encoding of the one-shot algorithm,
we have used the waiting time of the workers as the primary objective and the trajectory cost of the recharger as the secondary objective.
\shortversion{
We submit a video demonstrating the Two-shot algorithm as supplementary materials.
The implementation of our one-shot and two-shot algorithms are available at
{\color{blue}\url{https://www.dropbox.com/sh/zh43irn5sblwh2g/AACSS6rXxpRSsKpiF0XRnzAya?dl=0}}.
}

\subsection{Results}
\subsubsection{One-shot vs Two-shot algorithms}
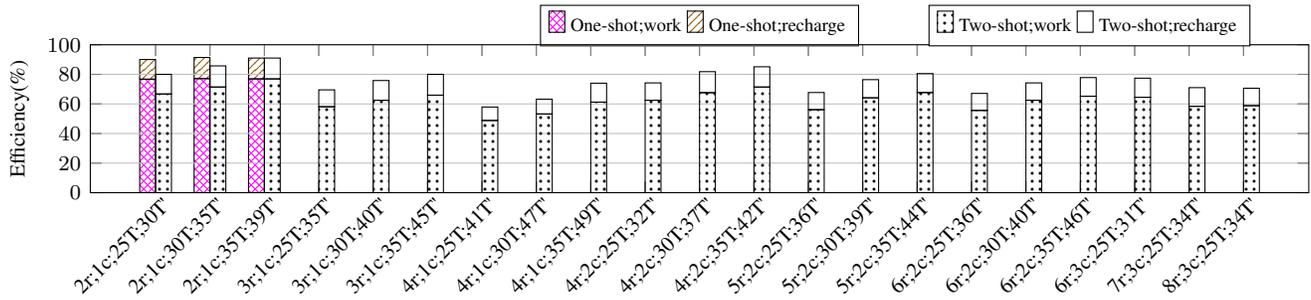
\begin{figure*}
\begin{tikzpicture}
\pgfplotsset{every tick label/.append style={font=\footnotesize}}
\begin{axis}[
    ybar stacked,
    bar shift=-3pt,
	bar width=6pt,
	axis line style={white},
	width  = \textwidth,
	height = 3.55cm,
    legend style={at={(0.5,1.27)},
    anchor=north,legend columns=-1},
    ylabel={\footnotesize Efficiency(\%)},
    symbolic x coords={2r;1c;25T;30T\textquotesingle, 2r;1c;30T;35T\textquotesingle, 2r;1c;35T;39T\textquotesingle,     3r;1c;25T;35T\textquotesingle, 3r;1c;30T;40T\textquotesingle, 3r;1c;35T;45T\textquotesingle,     4r;1c;25T;41T\textquotesingle, 4r;1c;30T;47T\textquotesingle, 4r;1c;35T;49T\textquotesingle,     4r;2c;25T;32T\textquotesingle, 4r;2c;30T;37T\textquotesingle, 4r;2c;35T;42T\textquotesingle,     5r;2c;25T;36T\textquotesingle, 5r;2c;30T;39T\textquotesingle, 5r;2c;35T;44T\textquotesingle,    6r;2c;25T;36T\textquotesingle, 6r;2c;30T;40T\textquotesingle, 6r;2c;35T;46T\textquotesingle,     6r;3c;25T;31T\textquotesingle, 7r;3c;25T;34T\textquotesingle, 8r;3c;25T;34T\textquotesingle},
    xtick=data,
    x tick label style={color=black, rotate=45, anchor=east, align=center, xshift=0.15cm, yshift=-0.2cm},
    ymin=0, ymax=100,
    enlarge x limits = 0.06,
    ]
    
    \addplot+[ybar, draw=black, pattern color=magenta, pattern=crosshatch] plot coordinates {(2r;1c;25T;30T\textquotesingle, 76.7) (2r;1c;30T;35T\textquotesingle, 77.1) (2r;1c;35T;39T\textquotesingle, 76.9)    (3r;1c;25T;35T\textquotesingle, 0.0) (3r;1c;30T;40T\textquotesingle, 0.0) (3r;1c;35T;45T\textquotesingle, 0.0)    (4r;1c;25T;41T\textquotesingle, 0.0) (4r;1c;30T;47T\textquotesingle, 0.0) (4r;1c;35T;49T\textquotesingle, 0.0)    (4r;2c;25T;32T\textquotesingle, 0.0) (4r;2c;30T;37T\textquotesingle, 0.0) (4r;2c;35T;42T\textquotesingle, 0.0)    (5r;2c;25T;36T\textquotesingle, 0.0) (5r;2c;30T;39T\textquotesingle, 0.0) (5r;2c;35T;44T\textquotesingle, 0.0)    (6r;2c;25T;36T\textquotesingle, 0.0) (6r;2c;30T;40T\textquotesingle, 0.0) (6r;2c;35T;46T\textquotesingle, 0.0)     (6r;3c;25T;31T\textquotesingle, 0.0) (7r;3c;25T;34T\textquotesingle, 0.0) (8r;3c;25T;34T\textquotesingle, 0.0)};
    \addplot+[ybar, draw=black, pattern color=yyyyellow, pattern=north east lines] plot coordinates {(2r;1c;25T;30T\textquotesingle, 13.3) (2r;1c;30T;35T\textquotesingle, 14.3) (2r;1c;35T;39T\textquotesingle, 14.1)    (3r;1c;25T;35T\textquotesingle, 0.0) (3r;1c;30T;40T\textquotesingle, 0.0) (3r;1c;35T;45T\textquotesingle, 0.0)    (4r;1c;25T;41T\textquotesingle, 0.0) (4r;1c;30T;47T\textquotesingle, 0.0) (4r;1c;35T;49T\textquotesingle, 0.0)    (4r;2c;25T;32T\textquotesingle, 0.0) (4r;2c;30T;37T\textquotesingle, 0.0) (4r;2c;35T;42T\textquotesingle, 0.0)    (5r;2c;25T;36T\textquotesingle, 0.0) (5r;2c;30T;39T\textquotesingle, 0.0) (5r;2c;35T;44T\textquotesingle, 0.0)    (6r;2c;25T;36T\textquotesingle, 0.0) (6r;2c;30T;40T\textquotesingle, 0.0) (6r;2c;35T;46T\textquotesingle, 0.0)     (6r;3c;25T;31T\textquotesingle, 0.0) (7r;3c;25T;34T\textquotesingle, 0.0) (8r;3c;25T;34T\textquotesingle, 0.0)};
    \legend{\scriptsize One-shot;work, \scriptsize One-shot;recharge}
\end{axis}

\begin{axis}[
    ybar stacked, 
    bar shift=3pt,
	bar width=6pt,
	axis line style={black},
	width  = \textwidth,
	height = 3.55cm,
    legend style={at={(0.82,1.27)},
    anchor=north,legend columns=-1},
    symbolic x coords={2r;1c;25T;30T\textquotesingle, 2r;1c;30T;35T\textquotesingle, 2r;1c;35T;39T\textquotesingle,     3r;1c;25T;35T\textquotesingle, 3r;1c;30T;40T\textquotesingle, 3r;1c;35T;45T\textquotesingle,     4r;1c;25T;41T\textquotesingle, 4r;1c;30T;47T\textquotesingle, 4r;1c;35T;49T\textquotesingle,     4r;2c;25T;32T\textquotesingle, 4r;2c;30T;37T\textquotesingle, 4r;2c;35T;42T\textquotesingle,     5r;2c;25T;36T\textquotesingle, 5r;2c;30T;39T\textquotesingle, 5r;2c;35T;44T\textquotesingle,    6r;2c;25T;36T\textquotesingle, 6r;2c;30T;40T\textquotesingle, 6r;2c;35T;46T\textquotesingle,     6r;3c;25T;31T\textquotesingle, 7r;3c;25T;34T\textquotesingle, 8r;3c;25T;34T\textquotesingle},
    xtick=data,
    ymin=0, ymax=100,
    ymajorgrids=true,
    xticklabels=\empty,
    yticklabels=\empty,
    enlarge x limits = 0.06,
    ]
        
    \addplot+[ybar, draw=black, pattern color=black, pattern=dots] plot coordinates {(2r;1c;25T;30T\textquotesingle, 66.7) (2r;1c;30T;35T\textquotesingle, 71.4) (2r;1c;35T;39T\textquotesingle, 76.9)     (3r;1c;25T;35T\textquotesingle, 58.1) (3r;1c;30T;40T\textquotesingle, 62.5) (3r;1c;35T;45T\textquotesingle, 65.9)      (4r;1c;25T;41T\textquotesingle, 48.8) (4r;1c;30T;47T\textquotesingle, 53.2) (4r;1c;35T;49T\textquotesingle, 61.2)    (4r;2c;25T;32T\textquotesingle, 62.5) (4r;2c;30T;37T\textquotesingle, 67.6) (4r;2c;35T;42T\textquotesingle, 71.4)    (5r;2c;25T;36T\textquotesingle, 56.1) (5r;2c;30T;39T\textquotesingle, 64.1) (5r;2c;35T;44T\textquotesingle, 67.7)    (6r;2c;25T;36T\textquotesingle, 55.6) (6r;2c;30T;40T\textquotesingle, 62.5) (6r;2c;35T;46T\textquotesingle, 65.2)     (6r;3c;25T;31T\textquotesingle, 64.5) (7r;3c;25T;34T\textquotesingle, 58.4) (8r;3c;25T;34T\textquotesingle, 58.8)};
    \addplot+[ybar, draw=black, pattern color=black, pattern=none] plot coordinates {(2r;1c;25T;30T\textquotesingle, 13.3) (2r;1c;30T;35T\textquotesingle, 14.3) (2r;1c;35T;39T\textquotesingle, 14.1)     (3r;1c;25T;35T\textquotesingle, 11.4) (3r;1c;30T;40T\textquotesingle, 13.3) (3r;1c;35T;45T\textquotesingle, 14.1)      (4r;1c;25T;41T\textquotesingle, 9.10) (4r;1c;30T;47T\textquotesingle, 10.0) (4r;1c;35T;49T\textquotesingle, 12.8)    (4r;2c;25T;32T\textquotesingle, 11.7) (4r;2c;30T;37T\textquotesingle, 14.2) (4r;2c;35T;42T\textquotesingle, 13.7)    (5r;2c;25T;36T\textquotesingle, 11.6) (5r;2c;30T;39T\textquotesingle, 12.3) (5r;2c;35T;44T\textquotesingle, 12.8)    (6r;2c;25T;36T\textquotesingle, 11.5) (6r;2c;30T;40T\textquotesingle, 11.7) (6r;2c;35T;46T\textquotesingle, 12.6)     (6r;3c;25T;31T\textquotesingle, 12.9) (7r;3c;25T;34T\textquotesingle, 12.6) (8r;3c;25T;34T\textquotesingle, 11.8)};
    \legend{\scriptsize Two-shot;work, \scriptsize Two-shot;recharge}
\end{axis}

\end{tikzpicture}
\caption{Efficiency comparison: One-shot vs Two-shot approach, for different \textit{workers-rechargers-orig.hypercycle-ext.hypercycle} ($|R|;|C|;T;T'$) combinations and Warehouse workspace. Each stacked bar is divided into work (lower) and recharge (upper) parts. For One-shot, in some cases, absence of bars implies timeout.}
\label{bar-onevstwoshot-effic}
\end{figure*}
\begin{figure} [t]
\begin{tikzpicture}
\pgfplotsset{every tick label/.append style={font=\footnotesize}}

\begin{semilogyaxis}[
restrict expr to domain={y}{1:1e4},
unbounded coords=discard,
x tick label style={
/pgf/number format/1000 sep=},
width  = 0.47*\textwidth,
height = 3.55cm,
ylabel={\footnotesize Time (minutes)},
xtick=data,
legend pos=north east,
ymajorgrids=true,
grid style=dashed,
enlargelimits=0.15,
enlarge x limits=0.11,
symbolic x coords={2r;1c, 3r;1c, 4r;1c, 4r;2c, 5r;2c, 6r;2c},
xtick = data,
legend style={at={(0.33,1.0)}, anchor=north,legend columns=-1},
]
 \addplot[mark=triangle*]
 coordinates {(2r;1c, 2.0) (3r;1c, 3.0) (4r;1c, 7.0) (4r;2c, 18.25) (5r;2c, 41.7) (6r;2c, 47.5)};

 \addplot[mark=*]
 coordinates {(2r;1c, 3.0) (3r;1c, 5.7) (4r;1c, 14.25) (4r;2c, 41.0) (5r;2c, 73.8) (6r;2c, 153.2)};

 \addplot[mark=square*]
 coordinates {(2r;1c, 4.4) (3r;1c, 10.7) (4r;1c,16.8) (4r;2c, 65.3) (5r;2c, 131.2) (6r;2c, 241.4)}; 

\legend{\scriptsize{T=25}, \scriptsize{T=30}, \scriptsize{T=35}}
\end{semilogyaxis}
\end{tikzpicture}
\caption{Computation time: Two-shot algorithm, for different \textit{workers-rechargers} combinations ($|R|;|C|$) and  hypercycle lengths (T) in warehouse workspace.}
\label{bar-twoshot-time}
\end{figure}
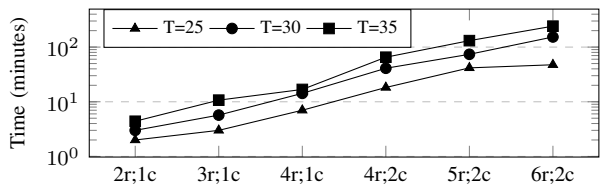
\begin{figure}
\begin{center}
\begin{tikzpicture}
\pgfplotsset{every tick label/.append style={font=\footnotesize}}

\pgfplotstableread{
76.0 14.0
0.0 0.0
0.0 0.0

0.0 0.0
0.0 0.0
0.0 0.0

0.0 0.0
0.0 0.0
0.0 0.0

0.0 0.0
0.0 0.0
0.0 0.0

0.0 0.0
0.0 0.0
0.0 0.0

0.0 0.0
0.0 0.0
0.0 0.0
}\oneshot

\pgfplotstableread{
69.2 13.5
53.5 11.0
50.0 9.8

64.5 14.5
56.1 11.1
52.6 10.1

55.4 11.6
52.7 11.3
48.4 10.2

57.8 12.5
54.7 15.9
50.9 10.2

45.2 8.6
38.0 7.3
41.1 8.3

59.7 12.9
51.7 10.0
50.9 10.7



}\twoshotpreemp

\begin{axis}[ 
    ybar stacked, bar shift=-1.77pt,
	axis line style={white},
    nodes near coords align={vertical},
    every node near coord/.append style={font=\footnotesize, fill=white, yshift=0.5mm},
    enlarge y limits={lower, value=0.1},
    enlarge y limits={upper, value=0.22},
    bar width=3.5pt,    
    xtick=data,
    ymin = 0, ymax=85,
    xticklabels={ 
    4r;2c, 5r;2c, 6r;2c,
    4r;2c, 5r;2c, 6r;2c,
    4r;2c, 5r;2c, 6r;2c,
    4r;2c, 5r;2c, 6r;2c,
    4r;2c, 5r;2c, 6r;2c,
    4r;2c, 5r;2c, 6r;2c,
    },
    legend style={at={(0.23, 1.29)}, anchor=north, legend columns=3},
    x tick label style={color=white, align=center, yshift=0.0cm},
    y tick label style={color=white},
    enlarge x limits=0.07,    
    width=9.0cm,
    height=3.45cm,
]
    \pgfplotsinvokeforeach {0,...,0}{  
        \addplot+[ybar, draw=black, pattern color=magenta, pattern=crosshatch] table [x expr={\coordindex-mod(\coordindex,3)/3}, y index=#1] {\oneshot}; 
    }
    \pgfplotsinvokeforeach {1,...,1}{  
        \addplot+[ybar, draw=black, pattern color=black, pattern=north east lines] table [x expr={\coordindex-mod(\coordindex,3)/3}, y index=#1] {\oneshot}; 
    }    
    \legend{\scriptsize 1-shot;work, \scriptsize 1-shot;rech}
\end{axis}    

\begin{axis}[ 
    ybar stacked, bar shift=1.75pt,
	axis line style={black},
    ylabel={\footnotesize Efficiency(\%)},
    nodes near coords align={vertical},
    every node near coord/.append style={font=\footnotesize, fill=white, yshift=0.5mm},
    enlarge y limits={lower, value=0.1},
    enlarge y limits={upper, value=0.22},
    bar width=3.5pt,    
    xtick=data,
    ymin = 0, ymax=85,
    xticklabel style={font=\scriptsize},
    xticklabels={ 
    4r;2c, 5r;2c, 6r;2c,
    4r;2c, 5r;2c, 6r;2c,
    4r;2c, 5r;2c, 6r;2c,
    4r;2c, 5r;2c, 6r;2c,
    4r;2c, 5r;2c, 6r;2c,
    4r;2c, 5r;2c, 6r;2c,    
    },
    legend style={at={(0.77, 1.29)}, anchor=north, legend columns=3},
    x tick label style={color=black, rotate=45, anchor=east, align=center, xshift=0.1cm, yshift=-0.2cm},
    y tick label style={color=black},
    enlarge x limits=0.07,    
    width=9.0cm,
    height=3.45cm,
]
    \pgfplotsinvokeforeach {0,...,0}{  
        \addplot+[ybar, draw=black, pattern color=black, pattern=dots] table [x expr={\coordindex-mod(\coordindex,3)/3}, y index=#1] {\twoshotpreemp}; 
    }
    \pgfplotsinvokeforeach {1,...,1}{  
        \addplot+[ybar, draw=black, pattern color=black, pattern=none] table [x expr={\coordindex-mod(\coordindex,3)/3}, y index=#1] {\twoshotpreemp}; 
    }
    \draw[very thick] (axis cs:5.35,0) -- ({axis cs:5.35,0}|-{rel axis cs:0.5,1});
    \draw[very thick] (axis cs:11.35,0) -- ({axis cs:11.35,0}|-{rel axis cs:0.5,1});
    
    \draw[very thin] (axis cs:2.32,0) -- ({axis cs:2.32,0}|-{rel axis cs:0.5,1});
    \draw[very thin] (axis cs:8.35,0) -- ({axis cs:8.35,0}|-{rel axis cs:0.5,1});
    \draw[very thin] (axis cs:14.35,0) -- ({axis cs:14.35,0}|-{rel axis cs:0.5,1});
  
    \legend{\scriptsize 2-shot;work, \scriptsize 2-shot;rech}
\end{axis}

\node [align=center,
    text width=4cm, inner sep=0.25cm] at (1.22cm, 1.72cm) {\textsc{\footnotesize{Art. floor}}};
\node [align=center,
    text width=4cm, inner sep=0.25cm] at (3.62cm, 1.72cm) {\textsc{\footnotesize{Random-20}}};
\node [align=center,
    text width=4cm, inner sep=0.25cm] at (6.08cm, 1.72cm) {\textsc{\footnotesize{Random-30}}};

\node [align=center,
    text width=1.5cm, inner sep=0.25cm] at (1.0cm, 1.42cm) {\textsc{\footnotesize{$T$=$20$}}};
\node [align=center,
    text width=1.5cm, inner sep=0.25cm] at (2.15cm, 1.42cm) {\textsc{\footnotesize{$T$=$25$}}};

\node [align=center,
    text width=1.5cm, inner sep=0.25cm] at (3.3cm, 1.42cm) {\textsc{\footnotesize{$T$=$20$}}};
\node [align=center,
    text width=1.5cm, inner sep=0.25cm] at (4.5cm, 1.42cm) {\textsc{\footnotesize{$T$=$25$}}};
    
\node [align=center,
    text width=1.5cm, inner sep=0.25cm] at (5.6cm, 1.42cm) {\textsc{\footnotesize{$T$=$20$}}};
\node [align=center,
    text width=1.5cm, inner sep=0.25cm] at (6.9cm, 1.42cm) {\textsc{\footnotesize{$T$=$25$}}};
\end{tikzpicture}
\end{center}
\caption{Efficiency comparison: One-shot vs Two-shot approach, for different types of workspaces -- Artificial floor, Random-20 and Random-30 with $T=20$ and $T=25$. Each stacked bar is divided into work and recharge parts. For One-shot, in most cases, absence of bars implies timeout ($3\si{\hour}$).}
\label{bar-effic-worksp-algo}
\end{figure}
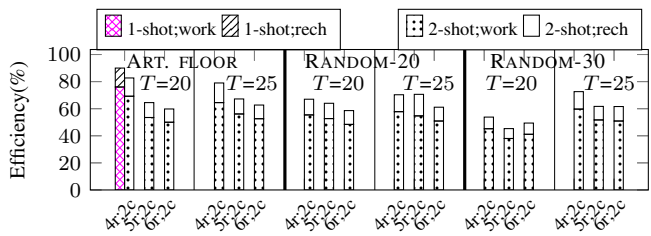
In Figure~\ref{bar-onevstwoshot-effic}, we compare the efficiency of the one-shot and the two-shot algorithms for Warehouse workspace. The one-shot approach does not have a concept of extended hypercycle ($T'$). However, the comparison needs to be done for the same length of hypercycles. Therefore, first, we execute the two-shot algorithm with some original hypercycle ($T$) and obtain the length of the extended hypercycle ($T'$). Subsequently, we use the derived $T'$ as the hypercycle length for the one-shot approach. 
The experiments are carried out for original hypercycle length $25$, $30$, and $35$ for up to $6$ workers and $2$ rechargers, and for original hypercycle length $25$ for $6$-$8$ workers and $3$ rechargers.
In the figure, the label on the $x$-axis 2r;1c;25T;30T\textquotesingle denotes $2$ workers, $1$ recharger, original hypercycle length $25$, and extended hypercycle length $30$.

For some smaller instances, the one-shot algorithm is able to produce recharge plans and provides better efficiency than its two-shot counterpart. Let us examine the reason.
Consider the instance $2r;1c;30T;35T'$ in Figure~\ref{bar-onevstwoshot-effic}. For one of the worker robots, the one-shot algorithm generates a recharge plan with $4$ working loops, whereas the two-shot algorithm allows only $3$ working loops. The reason behind this observation is that one of the workers completes its last working loop at time point $28$ in case of the one-shot algorithm. However, in the two-shot approach, the workers are not allowed to move after the $25$-th time point, as the original hypercycle length is $25$.
Therefore, one-shot gives better efficiency for this instance.
However, one-shot algorithm \textit{times out} ($3 \si{\hour}$) in most of the cases, whereas two-shot algorithm scales well for larger input instances. The computation times for the two-shot algorithm in warehouse workspace for up to $6$ workers and $2$ rechargers for original hypercycle length $25$, $30$,  and $35$ are shown in Figure~\ref{bar-twoshot-time}.

As evident from Figure~\ref{bar-onevstwoshot-effic}, with fixed $|R|$ and $|C|$, the efficiency increases with increasing $T$. With fixed $|R|$ and $T$, the efficiency increases with increasing $|C|$.

In Figure~\ref{bar-effic-worksp-algo}, we compare one shot and two shot approaches for other more complex workspaces -- Artificial floor, Random-20, and Random-30 (workspaces  are shown in Figure~\ref{fig:workspaces}). A similar trend, as seen in Figure~\ref{bar-onevstwoshot-effic} for Warehouse, can also be seen for other workspaces in Figure~\ref{bar-effic-worksp-algo}. One-shot algorithm times out for most of the input instances.



\subsubsection{Comparison with the Greedy algorithm}
We compare our SMT-based two-shot algorithm with the greedy algorithm presented in Section~\ref{sec-greedy}. Comparison results for $2-6$ workers and $1-2$ rechargers with hypercycle length $T=30$ and $T=35$ are shown in Figure~\ref{fig:optimizationGraph-ncs}.
As seen from the figure, our SMT-based two-shot algorithm achieves an improvement of $13$-$44\%$ in efficiency over the greedy algorithm, with an average $\%$-improvement of $27.5\%$. 
In the SMT-based two-shot algorithm, the possibility of partial recharging and synthesis of initial locations of rechargers helps us achieve better efficiency. 


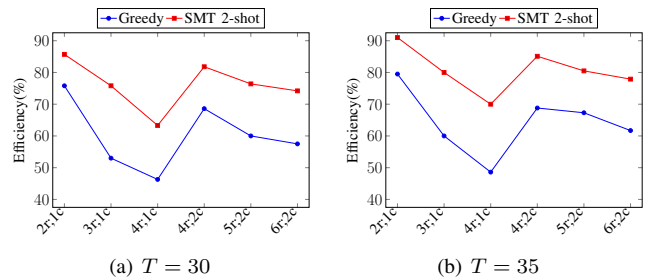
\begin{figure}[t]
\centering
{
\subfigure[$T=30$]{
\centering
\resizebox{0.47\linewidth}{!}
{
\begin{tikzpicture}
\pgfplotsset{every tick label/.append style={font=\LARGE}}
\begin{axis}[
	ylabel=\LARGE{Efficiency(\%)},
    symbolic x coords={2r;1c, 3r;1c, 4r;1c, 4r;2c, 5r;2c, 6r;2c},
    xtick=data,	
	enlargelimits=0.05,
	legend style={at={(0.5,1.14)},
	anchor=north,legend columns=-1},
	ymax=90, ymin=40,
    x tick label style={color=black, rotate=45, anchor=east, align=center, yshift=0.0cm},
    ymin=40, ymax=90,
    height=8cm,
    width=11cm,
]
\addplot
  coordinates { (2r;1c, 75.8) (3r;1c, 53.0) (4r;1c, 46.3) (4r;2c, 68.6) (5r;2c, 60.0) (6r;2c, 57.5)};
  \addlegendentry{\LARGE{Greedy}}
  
\addplot
  coordinates { (2r;1c, 85.7) (3r;1c, 75.8) (4r;1c, 63.3) (4r;2c, 81.8) (5r;2c, 76.4) (6r;2c, 74.2)};
  \addlegendentry{\LARGE{SMT 2-shot}}
\end{axis}
\end{tikzpicture}
}
}
\centering
\subfigure[$T=35$]{
\resizebox{0.47\linewidth}{!}
{
\begin{tikzpicture}
\pgfplotsset{every tick label/.append style={font=\LARGE}}
\begin{axis}[
	ylabel=\LARGE{Efficiency(\%)},
    symbolic x coords={2r;1c, 3r;1c, 4r;1c, 4r;2c, 5r;2c, 6r;2c},
	enlargelimits=0.05,
	legend style={at={(0.5,1.14)},
	anchor=north,legend columns=-1},
    x tick label style={color=black, rotate=45, anchor=east, align=center, yshift=0.0cm},
    ymin=40, ymax=90,
    height=8cm,
    width=11cm,
]
\addplot
  coordinates { (2r;1c, 79.5) (3r;1c, 60.0) (4r;1c, 48.6) (4r;2c, 68.8) (5r;2c, 67.3) (6r;2c, 61.7)};
  \addlegendentry{\LARGE{Greedy}}
  
\addplot
  coordinates { (2r;1c, 91.0) (3r;1c, 80.0) (4r;1c, 70.0) (4r;2c, 85.1) (5r;2c, 80.5) (6r;2c, 77.9)};
  \addlegendentry{\LARGE{SMT 2-shot}}
\end{axis}
\end{tikzpicture}
}}
}
\caption{Efficency comparison: SMT-based Two-shot vs Greedy approach, for different \textit{workers-rechargers} combinations and hypercycle lengths ($T$) in warehouse workspace.}
\label{fig:optimizationGraph-ncs}
\end{figure}


\shortversion {
We illustrate 
 the trajectories generated by the greedy algorithm and the SMT-based two-shot algorithm for a given instance in 
 in the full version of the paper~\cite{KunduS21}. 
 }
 
 \longversion {
 Figure~\ref{fig:greedy-vs-smt} (on the next page) shows the comparison between the greedy algorithm and the SMT-based two-shot algorithm for a given input instance ($4$ workers, $2$ rechargers, and $T=35$ in a Warehouse workspace). The efficiencies are $68.9\%$ and $81.5\%$ for the greedy approach and the SMT-based two-shot approach, respectively, which is an improvement of $18\%$ over the greedy approach. 
As Figure~\ref{fig:greedy-vs-smt} reveals, the SMT-based algorithm provides better performance for the following reasons.

\textit{i) The possibility of partial recharging in the SMT based approach results in the reduction in` the idle time for the worker robots.} 
In the case of the SMT-based approach, recharging is scheduled even if a worker is not fully out of charge. 
In many cases, partial recharging helps us gain performance improvement. 
Partial recharging happens when a worker is recharged for less number of time instants than what it actually needs to fully recharge its battery. Because of this, worker-2 and worker-4 traverses $4$ loops each with the greedy approach, whereas they traverse $5$ loops each with the SMT based two-shot approach. This enhances the efficiency of the SMT two-shot approach. Both worker-1 and worker-3 traverse $7$ loops each for both the approaches. 

\textit{ii) Synthesis of rechargers\textquotesingle \ initial locations in the SMT-based approach results in less time for the rechargers to go back to their initial states.} 
In the greedy approach, rechargers\textquotesingle\ initial locations are decided arbitrarily. Therefore, it takes a longer time for the rechargers to get back to their respective initial locations for state matching at the end. In Figure~\ref{fig:greedy-vs-smt}, length of the extended hypercycle for greedy approach is $45$. Had these initial locations been chosen wisely, overall time could have been saved. Our SMT-based two-shot algorithm synthesizes the initial locations of the rechargers, which leads to more convenient initial locations, and enables the rechargers to go back to their initial locations in comparatively lesser time. In Figure~\ref{fig:greedy-vs-smt}, the length of the extended hypercycle for the SMT-based two-shot algorithm is $42$ for the chosen input instance.

\begin{figure*}[t]
\begin{center}
\includegraphics[scale=0.88]{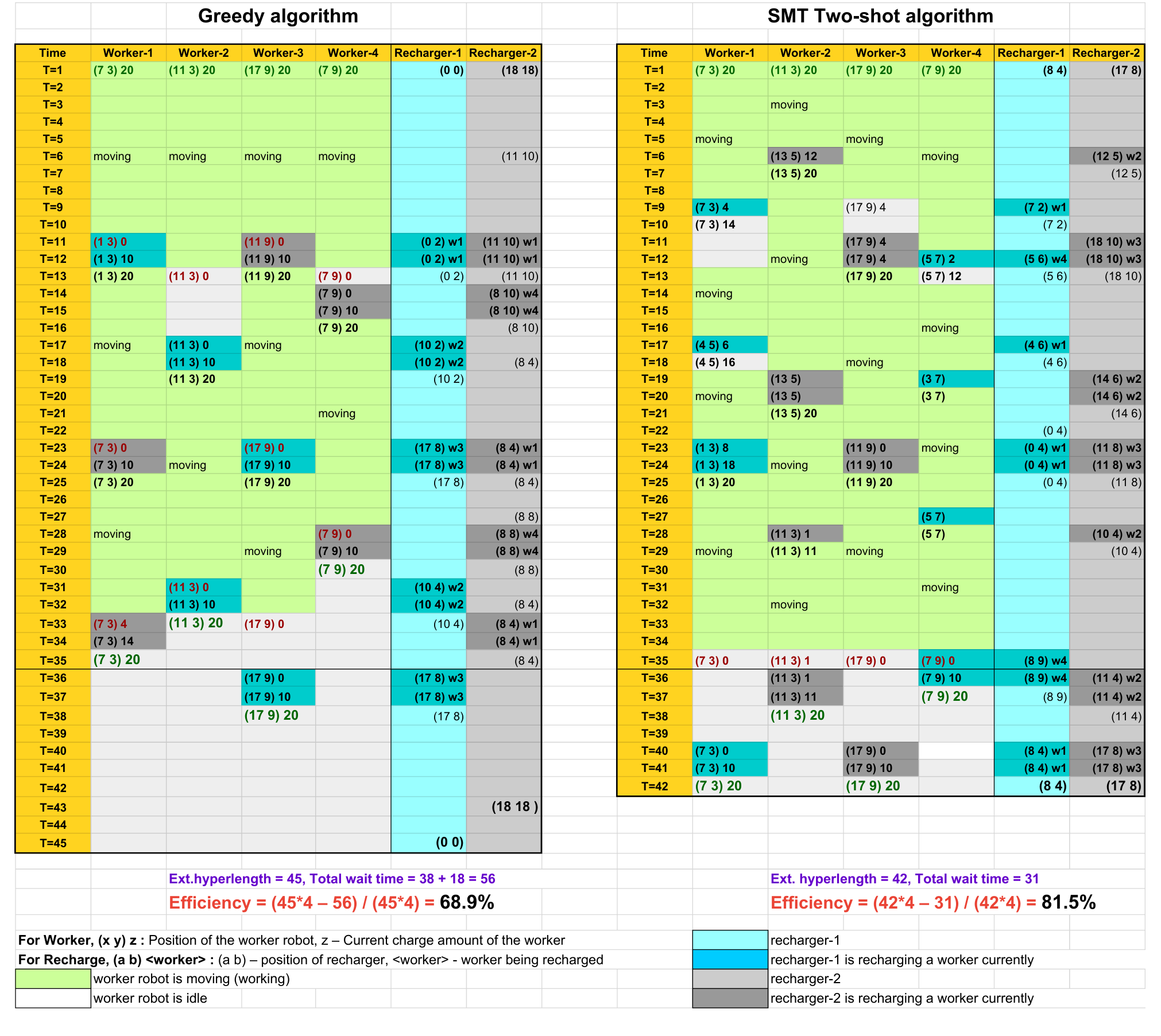}
\caption{Comparison of the Greedy algorithm and the SMT Two-shot algorithm for a given input instance (4 workers, 2 rechargers, $T=35$ and Warehouse workspace.)}
\label{fig:greedy-vs-smt}
\end{center}
\end{figure*}
}

\subsubsection{Effect of rechargers' potential initial locations ($P$)}

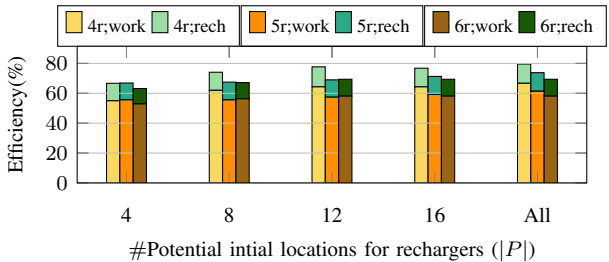
\begin{figure}
\begin{center}
\begin{tikzpicture}
\pgfplotsset{every tick label/.append style={font=\footnotesize}}

\pgfplotstableread{
55.0 11.7
62.0 12.1
64.3 13.4
64.3 12.5
66.7 12.9
}\fourworkers

\pgfplotstableread{
55.6 11.3
55.6 11.9
57.4 11.6
59.3 12.0
61.4 12.4
}\fiveworkers

\pgfplotstableread{
52.9 10.3
56.3 10.9
58.1 11.3
58.1 11.3
58.1 11.3
}\sixworkers

\begin{axis}[ 
    ybar stacked, bar shift=-5.05pt,
	width  = 0.47\textwidth,
	height = 3.4cm,    
	axis line style={white},
    nodes near coords align={vertical},
    every node near coord/.append style={font=\footnotesize, fill=white, yshift=0.5mm},
    enlarge y limits={lower, value=0.1},
    enlarge y limits={upper, value=0.22},
    bar width=5pt,    
    xtick=data,
    ymin=0, ymax=75,
    xticklabels={ 
    4, 8, 12, 16, All
    },
    legend style={at={(0.14, 1.30)}, anchor=north, legend columns=3},
    x tick label style={color=white, align=center, yshift=-0.2cm},
    y tick label style={color=white},
    enlarge x limits=0.12,    
]
    \pgfplotsinvokeforeach {0,...,0}{  
       \addplot+[style={black,fill=yyellow,mark=none}] table [x expr={\coordindex}, y index=#1] {\fourworkers};
    }
    \pgfplotsinvokeforeach {1,...,1}{  
       \addplot+[style={black,fill=ggreen,mark=none}] table [x expr={\coordindex}, y index=#1] {\fourworkers};
    }    
    \legend{\scriptsize 4r;work, \scriptsize 4r;rech}
\end{axis}

\begin{axis}[ 
    ybar stacked, bar shift=0pt,
	width  = 0.47\textwidth,
	height = 3.4cm,    
	axis line style={white},
    nodes near coords align={vertical},
    every node near coord/.append style={font=\footnotesize, fill=white, yshift=0.5mm},
    enlarge y limits={lower, value=0.1},
    enlarge y limits={upper, value=0.22},
    bar width=5pt,    
    xtick=data,
    ymin=0, ymax=75,
    xticklabels={ 
    4, 8, 12, 16, All
    },
    legend style={at={(0.50, 1.30)}, anchor=north, legend columns=3},
    x tick label style={color=white, align=center, yshift=-0.2cm},
    y tick label style={color=white},
    enlarge x limits=0.12,    
]
    \pgfplotsinvokeforeach {0,...,0}{  
       \addplot+[style={black,fill=yyyellow,mark=none}] table [x expr={\coordindex}, y index=#1] {\fiveworkers};        
    }
    \pgfplotsinvokeforeach {1,...,1}{  
       \addplot+[style={black,fill=gggreen,mark=none}] table [x expr={\coordindex}, y index=#1] {\fiveworkers};        
    }    
    \legend{\scriptsize 5r;work, \scriptsize 5r;rech}
\end{axis}

\begin{axis}[ 
    ybar stacked, bar shift=5.05pt,
	width  = 0.47\textwidth,
	height = 3.4cm,    
	axis line style={black},
    nodes near coords align={vertical},
    every node near coord/.append style={font=\footnotesize, fill=white, yshift=0.5mm},
    enlarge y limits={lower, value=0.1},
    enlarge y limits={upper, value=0.22},
    bar width=5pt,    
    xtick=data,
    xlabel= {\footnotesize $\#$Potential intial locations for rechargers ($|P|$)},
    ymin=0, ymax=75,
    ylabel={\footnotesize Efficiency(\%)},
    ymajorgrids=true,
    xticklabels={ 
    4, 8, 12, 16, All
    },
    legend style={at={(0.86, 1.30)}, anchor=north, legend columns=3},
    x tick label style={color=black, align=center, yshift=-0.2cm},
    y tick label style={color=black},
    enlarge x limits=0.12,    
]
    \pgfplotsinvokeforeach {0,...,0}{  
        \addplot+[style={black,fill=yyyyellow,mark=none}] table [x expr={\coordindex}, y index=#1] {\sixworkers}; 
    }
    \pgfplotsinvokeforeach {1,...,1}{  
        \addplot+[style={black,fill=ggggreen,mark=none}] table [x expr={\coordindex}, y index=#1] {\sixworkers}; 
    }    
    \legend{\scriptsize 6r;work, \scriptsize 6r;rech}
\end{axis}
\end{tikzpicture}
\end{center}
\caption{Efficiency vs Potential initial locations ($|P|$) for two rechargers and different numbers of workers with $T=20$ in Warehouse workspace. Each stacked bar is divided into work and recharge parts.}
\label{plot-potential}
\end{figure}

Our algorithm synthesizes the initial locations of the rechargers. The number of potential initial locations ($P$) of the rechargers plays a role in determining their starting positions, and hence the trajectories of the rechargers, and the overall efficiency (Figure~\ref{plot-potential}). 
We choose the locations in $P$ symmetrically in the Warehouse workspace as shown in Figure~\ref{fig:workspaces} for $|P|$=$16$. Similarly, we choose the locations for smaller $P$'s by symmetrically removing locations from the larger $P$.
As expected, the efficiency increases with the increase in the size of $P$ up to a certain value and then remains the same.
When we increase the size of set $P$, the computation time goes up because the search space size increases. Computation time varies in the range of $4\si{\minute}$-$6\si{\minute}$, $6\si{\minute}$-$10\si{\minute}$ and $9\si{\minute}$-$14\si{\minute}$ for $4$, $5$ and $6$ worker robots respectively. 


\shortversion {
\subsubsection{Effect of $\mathbf{\delta_{max}}$} Experimental result showing the effect of $\delta_{max}$ (maximum recharge amount per unit time) on the efficiency is provided in
the full version of the paper~\cite{KunduS21}.
}

\longversion {
\subsubsection{Effect of maximum recharge amount per unit time ($\delta_{max}$).}
\label{sec-delta-max}
The duration of a recharge instance depends on $\delta_{max}$.
Generally, higher $\delta_{max}$-value results in higher efficiency (Figure~\ref{plot-deltamax}) because the number of working loops for a worker is generally more with higher $\delta_{max}$ value. However, it may happen that for some sporadic cases, the efficiency drops for a higher $\delta_{max}$ value. The reason is that for a higher $\delta_{max}$ value, the recharge instances are likely to decrease but the number of working loops remains the same (or becomes slightly high), and we consider recharge instances to be useful work and thus contribute to efficiency of the worker.

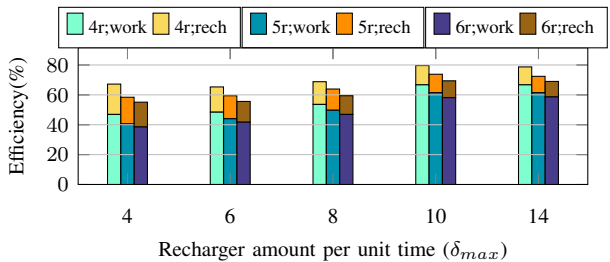
\begin{figure}[ht] 
\begin{center}
\begin{tikzpicture}
\pgfplotsset{every tick label/.append style={font=\footnotesize}}
\pgfplotstableread{
46.9 20.3
48.4 16.9
53.6 15.2
66.7 12.9
66.7 12.0
}\fourworkers
\pgfplotstableread{
40.5 17.9
44.0 15.4
49.7 14.2
61.4 12.4
61.4 11.0
}\fiveworkers
\pgfplotstableread{
38.5 16.6
41.7 13.9
46.9 12.5
58.1 11.3
58.6 10.4
}\sixworkers
\begin{axis}[ 
    ybar stacked, bar shift=-5.05pt,
	width  = 0.47\textwidth,
	height = 3.4cm,    
	axis line style={white},
    nodes near coords align={vertical},
    every node near coord/.append style={font=\footnotesize, fill=white, yshift=0.5mm},
    enlarge y limits={lower, value=0.1},
    enlarge y limits={upper, value=0.22},
    bar width=5pt,    
    xtick=data,
    ymin=0, ymax=75,
    xticklabels={ 
    4, 6, 8, 10, 14
    },
    legend style={at={(0.14, 1.30)}, anchor=north, legend columns=3},
    x tick label style={color=white, align=center, yshift=-0.2cm},
    y tick label style={color=white},
    enlarge x limits=0.12,    
]
    \pgfplotsinvokeforeach {0,...,0}{  
      \addplot+[style={black,fill=bblue,mark=none}] table [x expr={\coordindex}, y index=#1] {\fourworkers};
    }
    \pgfplotsinvokeforeach {1,...,1}{  
      \addplot+[style={black,fill=yyellow,mark=none}] table [x expr={\coordindex}, y index=#1] {\fourworkers};
    }    
    \legend{\scriptsize 4r;work, \scriptsize 4r;rech}
\end{axis}
\begin{axis}[ 
    ybar stacked, bar shift=0pt,
	width  = 0.47\textwidth,
	height = 3.4cm,    
	axis line style={white},
    nodes near coords align={vertical},
    every node near coord/.append style={font=\footnotesize, fill=white, yshift=0.5mm},
    enlarge y limits={lower, value=0.1},
    enlarge y limits={upper, value=0.22},
    bar width=5pt,    
    xtick=data,
    ymin=0, ymax=75,
    xticklabels={ 
    4, 6, 8, 10, 14
    },
    legend style={at={(0.50, 1.30)}, anchor=north, legend columns=3},
    x tick label style={color=white, align=center, yshift=-0.2cm},
    y tick label style={color=white},
    enlarge x limits=0.12,    
]
    \pgfplotsinvokeforeach {0,...,0}{  
      \addplot+[style={black,fill=bbblue,mark=none}] table [x expr={\coordindex}, y index=#1] {\fiveworkers};        
    }
    \pgfplotsinvokeforeach {1,...,1}{  
      \addplot+[style={black,fill=yyyellow,mark=none}] table [x expr={\coordindex}, y index=#1] {\fiveworkers};        
    }    
    \legend{\scriptsize 5r;work, \scriptsize 5r;rech}
\end{axis}
\begin{axis}[
    ybar stacked, bar shift=5.05pt,
	width  = 0.47\textwidth,
	height = 3.4cm,    
	axis line style={black},
    nodes near coords align={vertical},
    every node near coord/.append style={font=\footnotesize, fill=white, yshift=0.5mm},
    enlarge y limits={lower, value=0.1},
    enlarge y limits={upper, value=0.22},
    bar width=5pt,    
    xtick=data,
    xlabel= {\footnotesize Recharger amount per unit time ($\delta_{max}$)},
    ymin=0, ymax=75,
    ylabel={\footnotesize Efficiency(\%)},
    ymajorgrids=true,
    xticklabels={ 
    4, 6, 8, 10, 14
    },
    legend style={at={(0.86, 1.30)}, anchor=north, legend columns=3},
    x tick label style={color=black, align=center, yshift=-0.2cm},
    y tick label style={color=black},
    enlarge x limits=0.12,    
]
    \pgfplotsinvokeforeach {0,...,0}{  
        \addplot+[style={black,fill=bbbblue,mark=none}] table [x expr={\coordindex}, y index=#1] {\sixworkers}; 
    }
    \pgfplotsinvokeforeach {1,...,1}{  
        \addplot+[style={black,fill=yyyyellow,mark=none}] table [x expr={\coordindex}, y index=#1] {\sixworkers}; 
    }    
    \legend{\scriptsize 6r;work, \scriptsize 6r;rech}
\end{axis}
\end{tikzpicture}
\end{center}
\caption{Efficiency vs Maximum recharge amount per unit time ($\delta_{max}$) for different numbers of workers, with Warehouse workspace and $T=20$. Each stacked bar is divided into work and recharge parts.}
\label{plot-deltamax}
\end{figure}

In Figure~\ref{plot-deltamax}, such exceptions can be found for $4$ workers and $\delta_{max}=4 \text{ and } 6$. As shown in the stacked bar plot, efficiency decreases with $\delta_{max}=6$ because recharge instances drop in number, but the number of working loops remains the same, compared to $\delta_{max}=4$.
}




\section{Related Work}
\label{sec-related}
In this section, we discuss some related research work.

\smallskip
\noindent
\textit{Autonomous charging.}
For static rechargers, docking based autonomous recharging has been studied widely where the robots operate under fuel constraints~\cite{WawerlaJV2007,SundarR2014,StrimelV2014,YuAKR2015,MishraRMA2016,ShnapsR2016}.
Planning for recharge instants and locations is a good strategy to improve the overall performance.
To decide when and where to recharge a robot with static charging stations,~\cite{KannanMBD2013, Rappaport2017CoordinatedRO} leverage on a market-based strategy,~\cite{Aksaray2016DynamicRO} develops an approximate algorithm, and~\cite{TomyLHW19} schedule battery charge in the less busy period of the robots.


Energy management of robots employing mobile rechargers has also received attention from the robotics research community.
To service worker robots by static and mobile rechargers, recharge scheduling strategies are studied (\cite{sourabhral2019}), and a market-based strategy is proposed (\cite{KannanMBD2013}).
However, these work do not consider the path planning aspect of the mobile rechargers.
Litus et al.~\cite{Litus2007} introduce a method to find an optimal set of meeting places for a group of worker robots and a mobile refueling robot based on a given order of robot meetings.
Mathew et al.~\cite{Mathew_TRO_15} address a similar problem where the behaviors of the service robots are given by finite trajectories, repetition of which help perform some tasks persistently.
These papers discuss models for optimal-length path planning for the rechargers. In this paper, however, we discuss optimal recharge scheduling in terms of time and location to \emph{maximize the efficiency of the worker robots}, as well as \emph{minimize path length} for the mobile rechargers.


\smallskip
\noindent
\textit{SMT-based motion planning.}
To find optimal routes for the mobile rechargers to meet the worker robots at different locations is essentially an NP-hard problem~\cite{Vasile2014AnAA,scherercase2016}.
Our algorithmic solution to synthesize a recharge schedule for the worker robots and the trajectories for the recharger robots is based on a reduction of the problem to an SMT (Satisfiability Modulo Theory)~\cite{DeMoura:2011:SMT:1995376.1995394} solving problem. 
SMT solving allows us to solve NP-hard problems captured in the form of constraints expressed as decidable first-order logic formulas from different theories such as linear arithmetic or quantified Boolean logic.
SMT solvers are recently popular in solving the task and motion planning problems for robots~\cite{HungSTLZWG14,NedunuriPMCK14,SahaRKPS14,WangDCK16,SahaRKPS16,Shoukry2016ScalableLS,DesaiSYQS17,GavranMS17,camposicra2019}. SMT solver has been used recently for deciding optimal locations of the charging stations~\cite{KunduS18} and for energy-aware temporal logic motion planning for a mobile robot~\cite{KunduS19}. Imeson et al.~\cite{Imeson2018RoboticPP,Imeson2019AnSA} introduce a framework which couples task allocation and multi-robot motion planning using SMT solver.
To the best of our knowledge, we, for the first time, employ an SMT-based approach to solve the recharge planning problem for multi-robot systems with mobile rechargers.

\section{Conclusion}
\label{sec-conclusion}

In this paper, we present an SMT-based algorithm for solving the recharge scheduling and path planning problem for mobile rechargers that are responsible for supplying energy to the mobile robots involved in perpetual tasks. Though our algorithm is not optimal, it can solve complex planning problems in reasonable time and provide near-optimal solutions in terms of the efficiency of the workers. In the future,  we would  extend our framework to incorporate a more realistic battery model and evaluate our algorithm on a real multi-robot system. 




\bibliographystyle{IEEEtran}
\bibliography{ref,refrnc}

\end{document}